%% file: main_neurips_2026.tex
\documentclass{article}

\usepackage[numbers]{natbib}
\usepackage[main, final]{neurips_2026}

\usepackage[utf8]{inputenc} 
\usepackage[T1]{fontenc}    
\usepackage{hyperref}       
\usepackage{url}            
\usepackage{booktabs}       
\usepackage{amsfonts}       
\usepackage{nicefrac}       
\usepackage{microtype}      
\usepackage{xcolor}         
\usepackage{bm}
\usepackage{graphicx}
\usepackage{amsmath, amssymb, mathtools, amsthm}
\usepackage{xspace}
\usepackage{enumitem}
\usepackage{algorithm}
\usepackage{algpseudocode}

\usepackage{tabularx}
\usepackage{makecell}
\newcolumntype{Y}{>{\centering\arraybackslash}X}

\makeatletter
\newcommand*{\transpose}{%
	{\mathpalette\@transpose{}}%
}
\newcommand*{\@transpose}[2]{%
	\raisebox{\depth}{$\m@th#1\intercal$}%
}
\makeatother

\input{math_commands}

\newtheorem{theorem}{Theorem}[section]
\newtheorem{proposition}[theorem]{Proposition}

\usepackage[capitalize,noabbrev]{cleveref}
\crefname{lemma}{Lemma}{Lemmas}
\crefname{assumption}{Assumption}{Assumptions}
\crefname{prob}{Problem}{Problems}
\crefname{equation}{Eq.}{Eqs.}
\Crefname{equation}{Eq.}{Eqs.}
\crefname{ass}{Assumption}{Assumptions}
\crefname{assum}{Assumption}{Assumptions}

\newcommand{\true}{\mathrm{true}}

\newcommand{\LP}{\texttt{LP}\xspace}
\newcommand{\LPs}{\texttt{LP}s\xspace}

\title{Decision-Focused Learning in MDPs: \\An Occupancy Measure Approach}

\author{
\textbf{Zihao Zhao}$^{1}$\thanks{Corresponding author. Email: \texttt{zzhao628@gatech.edu}} \qquad
\textbf{Ashwath K. Karunakaram}$^{1}$ \\
\textbf{Ali Eshragh}$^{2,3}$\thanks{Equal supervision.} \qquad
\textbf{Yuexing Li}$^{2 \dagger}$ \qquad
\textbf{Kai Wang}$^{1\dagger}$ \\
[0.1em]
$^{1}$Georgia Institute of Technology, Atlanta, GA \\
$^{2}$Johns Hopkins Carey Business School, Washington, DC \\
$^{3}$International Computer Science Institute, Berkeley, CA \\
}

\begin{document}

\maketitle

\begin{abstract}
In this work, we consider decision-focused learning (DFL) for a Markov decision process (MDP), where existing methods differentiate through the KKT conditions of the Bellman equation and require solving a linear system over all state-action pairs, limiting its scalability.
We address this by reformulating the MDP as an occupancy measure-based linear program (\LP), whose feasible region is induced by predicted dynamics, and we derive a closed-form gradient by identifying the active constraints in the feasible polyhedron via the pivoting algorithm. 
This occupancy measure-based \LP layer raises two challenges: (1) \LP's solution gradient is discontinuous when active constraints change, and (2) the \LP backward cost still scales with the state size, which is costly for large or continuous state spaces.
We address the challenges with an augmented Lagrangian surrogate and smooth the boundary jumps by random row sketching of the constraints, and a learnable soft state-aggregation layer and its function-approximation generalization that scales the \LP to large finite and continuous-state MDPs.
Across multiple tasks, our methods reach lower regret than KKT-based DFL and two-stage baselines with significantly lower computation cost.
The source code for all experiments is available  \href{https://github.com/A-Eshragh/State_Aggregation_Project}{here}.
\end{abstract}

\input{01_introduction}

\input{02_related_works}

\input{03_problem_statement}
\input{04_dfl_mdp}
\input{05_experiments}

\input{06_conclusion}


\begin{ack}
We thank Wuhao Bao for detailed discussions of the experiments. This work was supported in part by the Johns Hopkins Carey Business School. This work was supported by NSF IIS-2403240, NSF IIS-2552007, NIH R01HL184139, and Schmidt Sciences AI2050 Fellowship.
\end{ack}

\bibliographystyle{plainnat}
\bibliography{references}

\newpage
\appendix
\input{appendix}

\newpage
\input{checklist.tex}

\end{document}

%% file: math_commands.tex
\usepackage{amsmath,amsfonts,bm, bbm}

\def\eqref#1{equation~\ref{#1}}

\def\1{\bm{1}}

\DeclareMathAlphabet{\mathsfit}{\encodingdefault}{\sfdefault}{m}{sl}
\SetMathAlphabet{\mathsfit}{bold}{\encodingdefault}{\sfdefault}{bx}{n}

\def\gA{{\mathcal{A}}}

\def\gN{{\mathcal{N}}}

\def\gS{{\mathcal{S}}}

\newcommand{\E}{\mathbb{E}}

\newcommand{\R}{\mathbb{R}}



%% file: 01_introduction.tex
\section{Introduction}\label{sec:introduction}
Many sequential decision-making problems require solving an optimization problem whose parameters are unknown in advance and must be predicted from contextual data, e.g., inventory control under uncertain demand~\cite{agrawal2019learning}. Standard two-stage pipelines train a predictive model to minimize forecast error and then feed the predictions into a downstream optimizer. However, minimizing prediction error does not necessarily optimize downstream decision quality~\cite{elmachtoub2022smart, mandiDecisionFocusedLearningLens2022}. Decision-focused learning (DFL) addresses this mismatch by training the predictive model end-to-end, backpropagating decision-quality gradients through the optimization layer~\citep{mandiDecisionFocusedLearningLens2022}.

When the downstream problem is a Markov decision process (MDP), the optimization solves the Bellman optimality equation from predicted rewards and transition probabilities~\citep{wang2021learning, futoma2020popcorn}. Existing DFL methods for MDPs typically differentiate through the Karush--Kuhn--Tucker (KKT) conditions of Bellman optimality, viewing the policy as the solution of a nonlinear system~\citep{wang2021learning}. This entails solving a linear system of dimension $|\gS||\gA|$, which is prohibitive for large state or action spaces. The Bellman operator's nonlinearity can also make this system poorly conditioned, especially as the discount factor approaches one. These issues limit the scalability of KKT-based DFL methods.

In this work, we propose an occupancy-measure \LP layer that overcomes both limitations by exploiting the \emph{linear programming} (\LP) formulation of MDPs over occupancy measures, whose feasible region depends on the predicted dynamics. The \LP has two key advantages. First, under a strictly positive initial distribution, the occupancy \LP is non-degenerate, yielding a well-defined optimal basis and avoiding KKT conditioning issues. Second, differentiating through the \LP layer requires inverting only a single $|\gS| \times |\gS|$ basis matrix, giving a closed-form, well-conditioned gradient and substantially lower cost, especially when $|\gA|$ is large.

However, directly using the uncertain feasible region \LP reformulation faces two practical challenges. 
(a) \emph{Biased loss minimum.} Standard constrained-DFL settings typically predict only objective coefficients, whereas in MDPs, predicted parameters enter both the \LP objective through rewards and the constraints through transitions. Thus, the predicted occupancy is feasible only under the predicted dynamics and generally violates the true flow constraints. Using the biased occupancy in the training loss shifts the loss minimum away from the true parameter, leading the predictor toward a biased target.
(b) \emph{Gradient discontinuity.} 
Changes in the optimal basis can cause sudden jumps in the loss and its gradients, which makes gradient-based training unstable~\citep{mmb+25, berthet2020learning, elmachtoub2022smart}. 
To address (a), we introduce an \emph{augmented Lagrangian surrogate} that adds both a linear dual correction and a quadratic penalty on the true-dynamics constraint to the training loss. 
With an optimal dual solution of the true \LP, we prove that the
true parameter is a global minimizer of this surrogate.
{To address (b), we apply \emph{random row sketching} to the flow constraints and average the resulting \LP solutions. Different sketches can select different bases, and averaging
empirically reduces abrupt gradient changes near basis boundaries.}

The $|\gS| \times |\gS|$ cost, however, is still costly when the state space is large or continuous. To further scale up, we introduce a soft state-aggregation layer that uses a learned softmax membership matrix to compress the $|\gS| \times |\gS|$ system to $M \times M$ with $M \ll |\gS|$. We extend the same construction to continuous-state MDPs via function approximation. Table~\ref{tbl:1} compares the resulting computational cost with the prior KKT approach~\citep{wang2021learning}.

\begin{table}[ht!]
    \centering
    \caption{\small{Comparison of computational complexity between the KKT approach, our LP approach, and our soft state-aggregation extension. Here $M$ denotes the number of aggregated states with $M \ll |\gS|$ and $\omega < 2.373$ denotes the matrix multiplication exponent.}}
    \vspace{-5pt}
    \label{tbl:1}
    \small
    \begin{tabular}{l|c|c|c}
        \textbf{Metric} &
        \textbf{KKT} &
        \textbf{\LP (ours)} &
        \textbf{\LP + State Agg (ours)} \\[0.1em] \hline
        Core Matrix to Invert &
        KKT System (\citep{wang2021learning}, Eq.(7)) &
        Optimal Basis $H_{\theta}^B$ &
        Aggregated Basis $\hat{H}_{\theta}^{B}$ \\[0.1em]
        Matrix Dimension &
        $|\mathcal{S}||\mathcal{A}| \times |\mathcal{S}||\mathcal{A}|$ &
        $|\mathcal{S}| \times |\mathcal{S}|$ &
        $M \times M$ \\
        Inversion Complexity &
        $\mathcal{O}((|\mathcal{S}||\mathcal{A}|)^{\omega})$ &
        $\mathcal{O}(|\mathcal{S}|^{\omega})$ &
        $\mathcal{O}(M^{\omega})$
    \end{tabular}
\end{table}


Our contributions can be summarized as follows:
\begin{itemize}[left=1em,nosep]
\item We propose a novel occupancy-measure \LP layer as a differentiable MDP layer with a closed-form, well-conditioned gradient requiring inversion of a $|\gS| \times |\gS|$ basis matrix.

\item We introduce an augmented Lagrangian DFL surrogate that augments the regret objective, with a global optimum guarantee, and employ random row sketching to empirically reduce sensitivity to basis changes. 

\item We extend the \LP layer to larger finite and continuous-state MDPs via learnable soft state aggregation and function approximation, and show competitive decision quality with lower training cost than the multiple baselines across three tasks
\end{itemize}

%% file: 02_related_works.tex
\section{Related Work}

\paragraph{Differentiable Optimization and Decision-Focused Learning.}

DFL trains a predictive model end-to-end with a downstream optimization layer so that the performance is measured by decision regret rather than parameter accuracy~\citep{elmachtoub2022smart, wilder2019melding, dontiTaskbasedEndtoendModel2017, mandiDecisionFocusedLearningLens2022}. The main challenge in DFL is differentiating through an argmin, which motivates many novel approaches, including implicit differentiation of KKT systems~\citep{ dontiTaskbasedEndtoendModel2017, amos2017optnet}; perturbation-, smoothing-, and interior-point-based relaxations of linear and combinatorial solvers~\citep{wdw+19, berthet2020learning, Pogan2020Differentiation, mandiInteriorPointSolving2020, paulusCombOptNetFitRight2021, dkf24}; surrogate losses with informative subgradients~\citep{elmachtoub2022smart, fwd+20, mmd+21, mandiDecisionFocusedLearningLens2022, tk24, mmb+25}; and learned task-aligned losses~\citep{shahDecisionFocusedLearningDifferentiable2022, jbp+25, kongEndtoEndStochasticOptimization2022}.

The closest line of work to ours considers DFL when uncertain parameters enter the \emph{constraints} of the downstream optimization problem~\citep{huPredictOptimizePackingCovering2023, hll23, mandi2026feasibilityaware, im2026smart, wang2023learning, cristianEndtoEndConstraintEnforcing2023}. These methods mainly focus on solving the issue that mispredicted constraints can make the true optimum infeasible and address this by post-hoc correction functions, feasibility-aware penalty losses, robust contextual uncertainty sets, or constraint-enforcing neural approximations. Their settings are often mixed-integer programs or contextual stochastic LPs without a sequential structure.
In contrast, our setting is the \emph{occupancy-measure linear program} of an MDP, in which a single prediction determines both the objective and constraints simultaneously through a predicted MDP and leads to a coupled sensitivity that prior constrained-DFL works do not have. To handle the predicted dynamics mismatch, we train the predictive model with an augmented Lagrangian surrogate that combines the downstream value under the true dynamics with a linear dual correction and a quadratic penalty on flow-constraint violation.
 
\paragraph{State Aggregation in Markov Decision Processes.}
State aggregation reduces the dimensionality of large MDPs by collapsing groups of states into abstract representations and is a classical tool in approximate dynamic programming and reinforcement learning~\cite{bertsekas1996neuro}. Prior approaches include homomorphism-based aggregation~\citep{ravindran2003smdp}, projection-based methods on a reduced basis~\citep{de2003linear}, deep representation-learning~\citep{gelada2019deepmdp}, and learnable soft aggregation with probabilistic membership~\citep{singh1994reinforcement,zhang2021learning}. We refer readers to~\Cref{app:related_works} for a detailed review. These methods, however, operate on a \emph{fixed} MDP. We instead differentiate a learnable state-aggregation layer end-to-end through a \emph{predicted} MDP, connecting representation learning with decision-focused learning.

%% file: 03_problem_statement.tex
\section{Problem Statement}\label{sec:prob_state}
We consider an end-to-end discounted MDP $(\gS, \gA, P_\theta, r_\theta, \beta, \gamma)$ with the state space $\mathcal{S}$, the action space $\mathcal{A}$, the discount factor $\beta \in [0,1)$ over an infinite horizon, and the initial state distribution $\gamma \in \R^{|\gS|}$ satisfying $\gamma(s) > 0$ for all $s \in \gS$. The transition kernel $P_\theta(s' \mid s, a)$ gives
the probability of moving to state~$s'$ from state~$s$ under action~$a$,
and $r_\theta(s,a)$ denotes the immediate reward in state~$s$ when action~$a$
is taken. Both the reward function and the transition kernel may depend on a
parameter vector $\theta \in \mathbb{R}^n$.
We  assume access to a training set $\{x_i, \theta^*_i\}_{i=1}^N$ in which each point $x_i$ is paired with a true MDP parameter $\theta^*_i$. Thus, the true reward $r_\true := r_{\theta^*}$ and the true transition kernel $P_\true := P_{\theta^*}$ are observed for every training instance and can be used to evaluate the regret. At inference time, only the input $x$ is available.


\paragraph{Occupancy-measure \LP} The discounted state-action occupancy measure $y \in \R^{|\gS||\gA|}$ is defined as $y(s,a) = \sum_{t=0}^\infty \beta^t \Pr(s_t = s, a_t = a)$, the expected discounted frequency of visiting state $s$ and taking action $a$ over the infinite horizon \citep{Eshragh_MOR_2011,Eshragh_MOR_2020}. It is well-known that solving MDPs can be equivalently written as solving the following \LP~\citep{manne1960linear, puterman2014markov}:
\begin{align}
    y^*_\theta = \arg\max_{y \geq 0} 
    r_\theta^\top y \quad
    \text{s.t.} \quad
    H_\theta y = \gamma. \tag{\texttt{LP-1}}
    \label{eq:lp_primal}
\end{align}
Here $H_\theta \in \mathbb{R}^{|\mathcal{S}| \times |\mathcal{S}||\mathcal{A}|}$
is the flow constraint matrix with entries
$[H_\theta]_{s,(s',a')} := \delta_{s,s'} - \beta P_\theta(s \mid s',a')$
where $\delta_{s,s'}$ is the Kronecker delta, i.e., $\delta_{s,s'} = 1$ if $s = s'$, and $0$ otherwise. The optimal policy can be recovered from $y^*_\theta$ by $\pi^*_\theta(a|s) = y^*_\theta(s, a)/\sum_{a' \in \gA} y^*_\theta(s, a')$.

In this paper, we focus on the DFL framework, where the goal is to learn a predictive model $m_{w}(x)$, parameterized by weights $w$, that maps input features $x$ to predicted MDP parameters ${\theta}$, in order to maximize the value (cumulative reward) under the true parameters. 

\paragraph{End-to-end learning objective}
A standard two-stage pipeline trains $m_w$ to minimize the prediction error between ${\theta}$ and the ground-truth $\theta^*$ using the loss $\|m_w(x) - \theta^*\|_2^2$, and solves the MDP at the inference time.
This pipeline ignores how prediction error translates into decision quality. 
On the other hand, DFL trains $m_w$ directly to maximize the quality of the resulting policy. The training objective of DFL can be written as a regret form:
\begin{align}
    \min_{w} L(w) := r_\true^\top y^*_\true - r_\true^\top y^*_\theta
    \label{eq:dfl_regret},
\end{align}
where $y^*_\true$ denotes the occupancy measure under the \emph{true} MDP $\theta^*$\footnote{The true regret is $V_{\true}(\pi^*)-V_{\true}(\pi_\theta)$. It is piecewise constant in $\theta$ because $\pi_\theta$ changes only when the optimal policy switches, making it difficult to differentiate directly.}. The first term is constant in $w$, so optimizing~\Cref{eq:dfl_regret} reduces to maximizing the value of the predicted policy $\hat{\pi}^*$.
By chain rule, we can write the gradient of $L(w)$ as:
\begin{align}
    \frac{d L}{dw}
    & =
    -r_\true^\top\frac{d y^*_\theta}{d \theta} \frac{d \theta}{dw}.
    \label{eq:lp_grad}
\end{align}
The predictor Jacobian $d\theta/dw$ is available through automatic
differentiation; the remaining challenges concern the \LP layer
and its training objective.

\paragraph{Why objective-only \LP methods do not directly apply.}
With fixed constraints, the optimal vertices are fixed. However, in our setting, the optimal vertices can move along with the parameter changes. Thus, previous smoothing methods derived for objective perturbations
over a fixed feasible region therefore require additional
analysis in our setting.

\textbf{(C1) Bias and gradient discontinuity of the LP layer.}  
(a) The predicted occupancy may violate the true flow constraints,
so improving its true-reward score need not improve the policy's
true return. The loss minimum can therefore shift away from the
true parameter.
(b) Basis changes can cause sudden jumps in the loss or its gradient.
We address (a) with an augmented Lagrangian surrogate and (b) with
random row sketching and solution averaging in~\Cref{sec:dfl_loss}.

\textbf{(C2) Computation cost of differentiating through the MDP.} 
Existing DFL methods differentiate through the KKT condition of the Bellman equation, treating the optimal policy as the solution of a \emph{nonlinear} system.  This involves inverting a matrix of dimension $|\gS||\gA| \times |\gS||\gA|$~\citep{wang2021learning, amos2017optnet}. This cost grows quickly with the action space and is the main obstacle to scaling DFL beyond small MDPs. We address this in~\Cref{sec:lp_grad} by deriving a closed-form LP-layer Jacobian whose main cost is solving an $|\gS| \times |\gS|$ linear system, and we further compress this to an $M \times M$ system ($M \ll |\gS|$) via soft state aggregation in~\Cref{sec:state_agg}.

%% file: 04_dfl_mdp.tex
\section{Differentiable MDP via Occupancy Measure LP}
We now show how to differentiate through the occupancy measure \LP. This replaces the nonlinear Bellman equation with linear constraints, and the resulting gradient depends only on a small basis matrix. Then, we will address the bias and gradient discontinuity challenges mentioned above.

\subsection{Closed-form Gradient of the \LP Layer} \label{sec:lp_grad}

Differentiating $y^*_\theta$ requires the optimal basis to be both \emph{unique} and \emph{locally constant} in $\theta$\footnote{{We do not require the \LP to have a globally unique optimum. When multiple optima exist, we follow~\citet{Pogan2020Differentiation} and treat the solver as a deterministic tie-breaking operator that returns a basis solution.}}. When the $\gamma(s) > 0$ for all $s$, every basis feasible solution is non-degenerate~\citep{altman2021constrained, puterman2014markov}. Thus, the positive entries of the selected solution uniquely determine an invertible basis matrix $H^B_\theta \in \R^{|\gS| \times |\gS|}$. We further assume this basis is locally stable, so small perturbations in $\theta$ do not change the selected basis. Under this assumption, the gradient derived below is exact within each basis region. 


In order to compute the Jacobian $dy^*_\theta/d\theta$, we exploit the relationship between $y_{\theta}^*$ and $\theta$ defined by the linear constraints.
Assume that the initial distribution $\gamma$ is independent of $\theta$, and therefore
\begin{align*}
    \frac{\partial H_{\theta}}{\partial \theta} y_{\theta}^*
    + H_{\theta} \frac{d y_{\theta}^*}{d \theta}
    = 0.
\end{align*}
Partitioning $y^*_\theta = (y^B_\theta, y^N_\theta)$ and $H_\theta = (H^B_\theta, H^N_\theta)$
into basic and non-basic components, non-degeneracy implies that $y^B_\theta \in \R^{|\gS|}$, $y^N_\theta = 0$ with $H^B_\theta \in \R^{|\gS| \times |\gS|}$ invertible.
Under the fixed basis assumption, we establish
\begin{align}
    \frac{d y_{\theta}^B}{d \theta}
    & = - \left(H_{\theta}^B\right)^{-1}
    \left( \frac{\partial H_{\theta}^B}{\partial \theta} y_{\theta}^B \right).
    \label{eq:jaco}
\end{align}
Applying the chain rule to the LP-layer Jacobian in~\Cref{eq:jaco} yields the gradient of any downstream loss $L$ that depends on $\theta$ only through $y^*_\theta$:
\begin{align}
    \frac{d L}{d\theta}
    &=
    -\left(\nabla_{y^B_\theta} L\right)^\top
    \left( H_{\theta}^{B} \right)^{-1}
    \left(
        \frac{\partial H_{\theta}^B}{\partial \theta}
         y_{\theta}^B
    \right).
    \label{eq:dfl_gradient}
\end{align}
Differentiating through \LPs is challenging for prior KKT-based DFL work because the LP solution has zero gradient. Specifically, when $\theta$ enters only the LP objective, $y^*_\theta$ does not depend on $\theta$ within a fixed basis region, so $dy^*_\theta/d\theta \equiv 0$ and KKT differentiation provides no learning signal. Common remedies use surrogates such as quadratic regularization on the objective, which distorts the optimum and increases per-step cost. In our occupancy-measure formulation $\theta$ enters both the objective and the flow constraints $H_\theta$, and the constraint Jacobian $\partial H^B_\theta/\partial\theta$ is generically non-zero. Thus, ~\Cref{eq:dfl_gradient} gives a non-zero closed-form gradient, which makes directly differentiating through the \LP possible.

\subsection{DFL Training Loss}\label{sec:dfl_loss}
Although the LP helps us compute a closed-form gradient, it also induces some issues.
Specifically, the original DFL objective ($r_\true^\top y^*_\theta$) can be biased. Since the predicted occupancy $y^*_\theta$ is feasible only under the predicted dynamics $H_\theta$, it generically violates the true flow constraints $H_{\mathrm{true}}$, and can therefore score higher on $r_{\mathrm{true}}^\top y$ than any truly $H_{\mathrm{true}}$-feasible occupancy. Minimizing $-r_{\mathrm{true}}^\top y^*_\theta$ over $\theta$ thus need not recover $\theta_{\mathrm{true}}$, as shown in~\Cref{fig:motivating} (1).
\begin{figure}[htbp]
    \centering
    \includegraphics[width=0.95\linewidth]{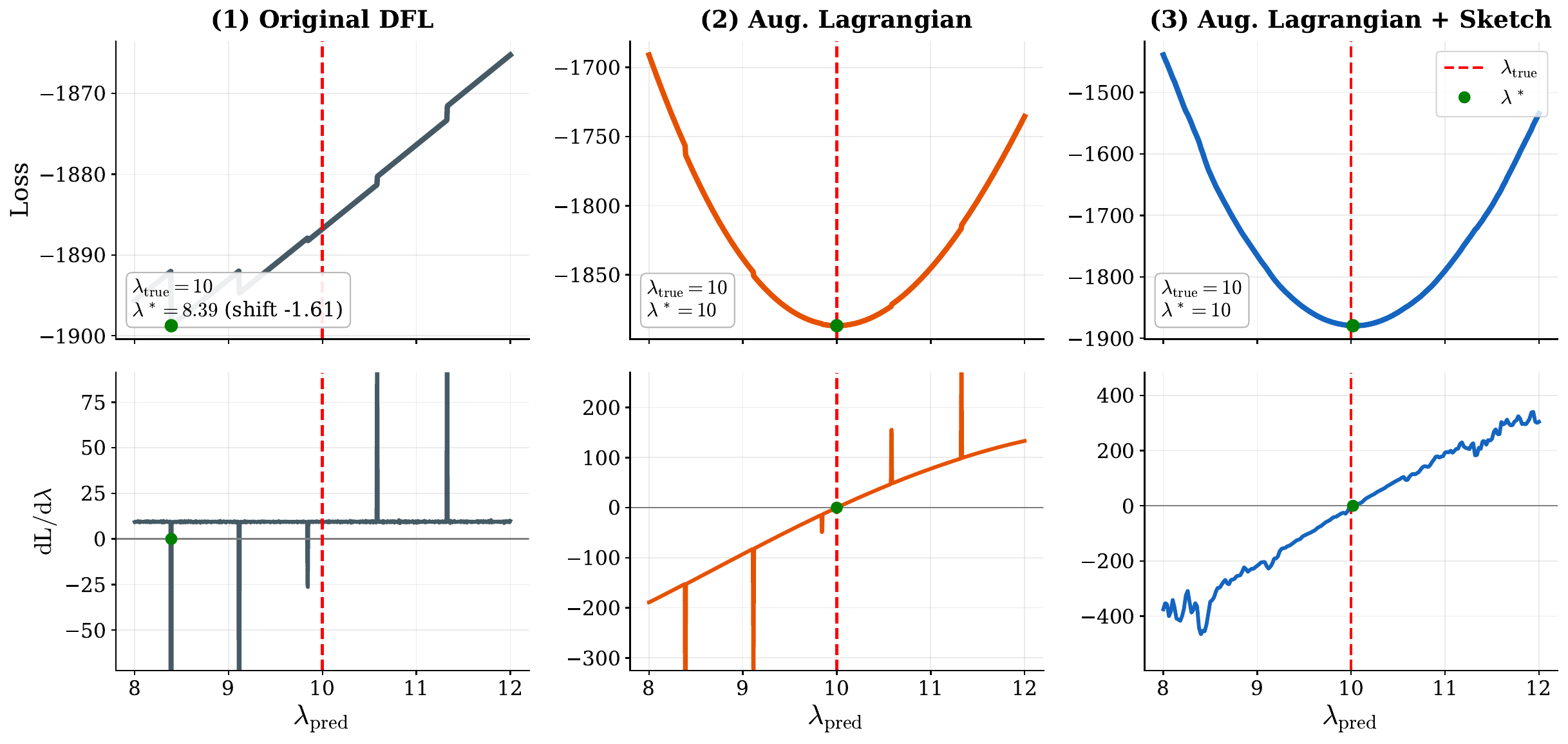}
    \vspace{-10pt}
    \caption{Loss (top) and gradient (bottom) of three DFL objectives as $\lambda_{\mathrm{pred}}$ varies; red dashed indicate $\lambda_\true$ and green dots indicate the loss minimum. (1) The occupancy proxy has a shifted minimum. (2) The augmented Lagrangian has a minimum at the true parameter. (3) Sketch averaging empirically reduces boundary sensitivity.}
    \label{fig:motivating}
\end{figure}


We address these issues with an \emph{augmented Lagrangian surrogate} that combines a decision-quality term with a {linear dual term} and a weighted feasibility penalty:
\begin{align}
L(y^*_\theta) := -r_{\mathrm{true}}^\top y^*_\theta + (H_{\mathrm{true}} y^*_\theta - \gamma)^\top\nu^* + \rho \|H_{\mathrm{true}} y^*_\theta - \gamma\|^2,
\label{eq:dfl_penalty}
\end{align}
where $\nu^*$ is an optimal dual solution of the true LP. {Note that no additional optimization is required, because training already solves the true LP, and standard LP solvers return the primal and dual solutions together.}

\begin{proposition}\label{prop_lp}
    Let $y^*_{\mathrm{true}}$ be the optimal occupancy for the true MDP and $\nu^*$ be its optimal dual.
    For any penalty $\rho > 0$,~\Cref{eq:dfl_penalty} satisfies $L(y^*_\theta) - L(y^*_{\mathrm{true}}) \geq 0$ for every predicted parameter $\theta$, with equality iff $y^*_\theta$ is optimal for the true LP. In particular, $\theta = \theta^*$ is a global minimizer.
\end{proposition}
See Appendix~\ref{app:pf_prop_lp} for proof. The surrogate evaluates the predicted occupancy $y^*_\theta$ under the true reward, while the dual and quadratic terms penalize violations of the true flow constraints. The quadratic penalty additionally strengthens the gradient as $y^*_\theta$ moves farther from the true-feasible set, which provides a stronger correction than linear regret (\Cref{fig:motivating} (2)). However, the loss remains discontinuous at basis boundaries. We therefore apply our sketch-and-solve smoothing method to obtain informative gradients across these transitions, which empirically reduces these jumps.


\paragraph{Sketch-and-solve} 
Classic perturbation-based smoothing for nonsmooth LPs~\citep{berthet2020learning} adds noise to the objective coefficients so that the random optimum spreads across vertices. Here we instead randomize the constraints, whose coefficients also depend on the predicted parameters. We then add randomness on the constraints through \textbf{random row sketching}~\citep{woodruff2014sketching, pilanci2016iterative} of $H_\theta$ and averaging the resulting LP solutions across sketches.
The key idea is that the average solution $\bar{y}^*_\theta$ moves off a single vertex and empirically makes the gradient less sensitive to basis changes.

Specifically, at each forward pass, we draw $K$ independent Gaussian sketches $S^{(1)},\dots,S^{(K)} \in \R^{k \times |\gS|}$ with $k = \alpha |\gS|$ rows, where $\alpha \in (0, 1]$ is the row-keep ratio. {To keep the Euclidean geometry, we draw $S^{(i)}$ with i.i.d. entries $S^{(i)}_{ab} \sim \gN(0, 1/k)$.}
Because $S^{(i)} H_\theta$ has fewer rows than $H_\theta$, the sketched feasible set is larger than the original and may contain unbounded directions, which can make the sketched \LP unbounded. We avoid this issue by keeping the mass-conservation identity $\mathbf{1}^\top y = 1/(1-\beta)$, which follows from summing all rows of $H_\theta y=\gamma$, as an extra constraint.
For each sketch $S^{(i)}$, we solve the \emph{sketched LP}
\begin{align}\label{eq:sketch_lp}
    y^{*(i)}_\theta = \arg\max_{y \geq 0} r_\theta^\top y
    \quad \text{s.t.} \quad
     S^{(i)} H_\theta y = S^{(i)} \gamma, \quad \mathbf{1}^\top y=\frac{1}{1-\beta}\quad(k<|\gS|)
\end{align}
using our differentiable LP layer.For $k=|\gS|$, the Gaussian sketch is invertible almost surely and the redundant mass row is omitted. We average the solutions, $\bar y^*_\theta=K^{-1}\sum_i y^{*(i)}_\theta$. On regions where all selected bases are stable, $d\bar y^*_\theta/d\theta=K^{-1}\sum_i dy^{*(i)}_\theta/d\theta$. The average need not satisfy the original flow constraints, and finite averaging does not guarantee smoothness. The preceding true-parameter minimum guarantee therefore does not automatically extend to this sketched loss. Sketching is used during training; inference uses the unsketched predicted problem.

\section{State Aggregation} \label{sec:state_agg}
In the occupancy-measure \LP, the basis matrix has dimension $|\gS| \times |\gS|$ compared to the KKT-based DFL method's $|\gS||\gA| \times |\gS||\gA|$, so the gradient computation reduces to inverting a much smaller matrix when the action space is large. However, it remains the bottleneck when the state space is large or continuous. We address this with a soft state-aggregation layer that is trained jointly with the predictor and extended to the general state MDPs through linear function approximation.

\paragraph{Operator form of the occupancy LP} Let $\gS$ be a measurable state space (finite or continuous) with initial distribution $\gamma$ and transition kernel $P_\theta$. The discounted occupancy $y : \gS \times \gA \to \R_{\geq 0}$ satisfies the (possibly infinite-dimensional) LP
\begin{align*}
    \max_{y \geq 0} \int r_\theta(s, a) y(s, a) ds da
    \ \ \ \text{s.t.} \ \ \
    \int y(s, a) da - \beta \int P_\theta(s | s', a') y(s', a') ds' da' = \gamma(s), \forall s \in \gS,
\end{align*}
where the integrals collapse to sums in the finite case. We abbreviate the flow constraint as $H_\theta y = \gamma$ below, with the flow operator $H_\theta$ taking $y : \gS \times \gA \to \R$ to the function on $\gS$ above. This generalizes the flow constraint to the continuous setting.

We approximate the occupancy by a finite expansion in $d$ state-action occupancy features $\Phi := [\phi_1, \ldots, \phi_d] : \gS \times \gA \to \R^d$, and project the flow constraint, which lives in $\gS$, onto $m$ state-side constraint features $\Psi := [\psi_1, \ldots, \psi_m] : \gS \to \R^m$:
\begin{align*}
    y_w(s, a) = \Phi(s, a)^\top w = \sum_{j=1}^d w_j \phi_j(s, a),
    \qquad
    \Psi^\top (H_\theta \Phi w - \gamma) = 0.
\end{align*}
If the occupancy features $\phi_j$ are non-negative, then $w \geq 0$ enforces $y_w \geq 0$, and we obtain the projected \LP
\begin{align}
    \max_{w \geq 0} \quad (\Phi^\top r_\theta)^\top w
    \quad \text{s.t.} \quad
    \Psi^\top H_\theta \Phi w = \Psi^\top \gamma,
    \tag{\texttt{LP-2}} \label{eq:projected_lp}
\end{align}
where $\Phi^\top$ and $\Psi^\top$ denote inner products over $\gS \times \gA$ and $\gS$, respectively.~\ref{eq:projected_lp} has the same form as~\ref{eq:lp_primal} under the substitutions $r_\theta \to \Phi^\top r_\theta$, $H_\theta \to \Psi^\top H_\theta \Phi$, $\gamma \to \Psi^\top \gamma$, so the closed-form gradient~\Cref{eq:dfl_gradient} applies directly: differentiating through~\ref{eq:projected_lp} requires inverting only the $m \times m$ optimal-basis submatrix of $\Psi^\top H_\theta \Phi$, dropping the per-step gradient cost from $\mathcal{O}(|\gS|^\omega)$ to $\mathcal{O}(m^\omega)$ with $m \ll |\gS|$. To recover an occupancy in the original $(\gS, \gA)$ space, we lift the projected solution by $y_{\text{disagg}}(s, a) = \Phi(s, a)^\top w^*_\theta$, where $w^*_\theta$ denotes the optimal solution for~\ref{eq:projected_lp}. Note that $\Phi$ and $\Psi$ are themselves \emph{learnable} and trained jointly with the predictor by backpropagating through~\Cref{eq:dfl_gradient}.

\subsection{Soft State Aggregation for Finite State MDPs}
For a finite state space, soft aggregation corresponds to a particular choice of $(\Phi, \Psi)$ derived from a \emph{learnable membership matrix} $W \in \R^{|\gS| \times M}$ with simplex rows ($W \geq 0$, $\sum_m W_{sm} = 1$ for all $s$). Soft assignments, in contrast to hard clustering, keep $W$ differentiable so that the aggregation structure is trained jointly with the predictor. Indexing the occupancy features by $j = (m, a')$ for $m \in [M]$ and $a' \in \gA$, we set
\begin{align}
    \phi_{(m, a')}(s, a) = W_{sm} \mathbf{1}[a = a'],
    \qquad
    \psi_i(s) = W_{si}.
    \label{eq:soft_agg_basis}
\end{align}
Since $W \geq 0$, non-negative coefficients $w$ yield a non-negative occupancy. However, direct projection does \textbf{not preserve} the
standard MDP flow constraints. For soft memberships, we replace $W^\top W$ with $\operatorname{diag}(\mu)$ to obtain a valid aggregated MDP, giving:
\begin{align}\label{eq:aggregated_mdp}
\hat r_{m,a'}=\frac{\sum_s W_{sm}r_{s,a'}}{\mu_m},\quad
\hat P(n'|m,a')=\frac{\sum_{s,s'}W_{sm}P(s'|s,a')W_{s'n'}}{\mu_m},\quad
\hat\gamma_m=\sum_s W_{sm}\gamma_s,
\end{align}
with cluster mass $\mu_m := \sum_s W_{sm}$ and aggregated flow constraint $\hat H_{i,(m,a')} = \delta_{i, m} - \beta \hat P(i | m, a')$. The projection collapses the $|\gS| \times |\gS|$ basis matrix to an $M \times M$ matrix with $M \ll |\gS|$, reducing the per-step gradient cost to $\mathcal{O}(M^\omega)$. Given an aggregated solution $\hat y$, we lift it to the
original state space as $y_{\text{disagg}}(s, a) = \sum_m W_{sm} w^*_{(m, a)}$. See~\Cref{app:agg_recovery} for more discussions.

\begin{algorithm}[htbp]
\caption{Decision-focused learning with state aggregation.}
\label{alg:main}
\begin{algorithmic}[1]
\Require training data $\mathcal D$, predictor $m_w$,
aggregation parameters $W_\xi$, number of aggregate states $M$,
step size $\eta$, penalty $\rho$, number of sketches $K$,
compression ratio $\alpha$
\State Precompute the optimal true-MDP dual solutions for the training data
\For{$t=1,2,\dots$}
  \State Sample $(x,\theta^*)\sim\mathcal D$;
  predict $\theta\gets m_w(x)$, construct $r_\theta,P_\theta$
  \Statex \hspace{\algorithmicindent}\textit{\# Reduce the LP dimension}
  \State $(\hat r,\hat P,\hat\gamma)\gets
  \Call{WeightedAggregate}{r_\theta,P_\theta,\gamma,W_\xi}$
  \Comment{\Cref{eq:aggregated_mdp}}
  \Statex \hspace{\algorithmicindent}\textit{\# Mitigate gradient jumps}
  \State $\hat y^{*(1)},\dots,\hat y^{*(K)}\gets
  \Call{SketchAndSolve}{(\hat r,\hat P,\hat\gamma),K,\alpha}$
  \Comment{\Cref{eq:sketch_lp}}
  \State $\bar{y}\gets
  K^{-1}\sum_{j=1}^{K}\hat y^{*(j)}$
  \Comment{average sketched sol.}
  \Statex \hspace{\algorithmicindent}\textit{\# Correct loss bias}
  \State $\ell\gets
  \Call{AugLagrangianWithLift}{\bar{y},W_\xi,\theta^*,\nu^*,\rho}$
  \Comment{\Cref{eq:dfl_penalty}}
  \State $(g_w,g_\xi)\gets\Call{BasisBackward}{\ell}$
  \Comment{\Cref{eq:jaco} + autograd}
  \State $w\gets w-\eta g_w$;
  $\xi\gets\xi-\eta g_\xi$
\EndFor
\State At inference, solve without sketching,
lift the occupancy, and normalize over actions
\end{algorithmic}
\end{algorithm}

\subsection{Function Approximation for General State MDPs}
For continuous state spaces, we use neural networks~\citep{hu2025primal}
to represent occupancy features $\Phi$ and test functions $\Psi$,
yielding a finite-dimensional LP.
We learn these features by backpropagating through the LP
with our basis-based gradient.

The projected LP coefficients involve state-space integrals
that generally have no closed form.
We estimate them from offline transitions
$\{(s_t,a_t,r_t,s_{t+1})\}_{t=1}^{N}$
and initial states $\{s_0^{(k)}\}_{k=1}^{N_0}$ drawn from $\gamma$.
In particular, the transition term is estimated as
\begin{align*}
    (\Psi^\top P\Phi)_{ij}
    \approx \frac{1}{N}\sum_{t=1}^{N}
    \psi_i(s_{t+1})\phi_j(s_t,a_t).
\end{align*}
Each transition pairs an occupancy feature at the current
state-action pair with a test function at the observed next state.
This replaces an expectation over the unknown transition dynamics
with a sample average, so no learned transition model is required.
The remaining coefficients and sampling convention are given
in \Cref{app:projected_lp}.
We train $\Phi$ and $\Psi$ by backpropagating through these
estimates and the LP using our basis-based gradient.




%% file: 05_experiments.tex
\section{Experiments}\label{sec:exp}
We evaluate our methods on three tasks that cover both finite and general-state MDPs. For all tasks, we compare the following methods:
\begin{itemize}[left=1em, nosep]
    \item \textbf{Two-stage}: a baseline that trains the predictive model with MSE on the MDP parameters and then solves the MDP at inference; no decision signal is passed back to the predictor.
    \item \textbf{DFL-QP}: the existing KKT-based DFL gradient~\citep{wang2021learning}, obtained by differentiating through a quadratic-programming relaxation of the Bellman optimality conditions, with the QP solved by \texttt{qpth}~\citep{amos2017optnet}. The relaxation is realized by adding a quadratic regularizer $\eta\|y^*_\theta\|^2$ to the \LP objective.
    \item \textbf{DFL-LP}: our closed-form gradient~(\Cref{eq:dfl_gradient}) derived from the occupancy-measure \LP.\
    \item {\textbf{DFL-Feas}: the feasibility-aware constrained-DFL baseline of \citet{mandi2026feasibilityaware}, adapted to our occupancy \LP, which penalizes violations of the true flow constraints during training}.
    \item \textbf{DFL-Sketch}: DFL-LP augmented with random row sketching and averaging of the flow constraints. This needs to solve multiple instances, but they can be easily parallelized, and each one only uses a small fraction of constraints (e.g., 10\%). 
    \item \textbf{DFL-SA}: DFL-Sketch combined with soft state aggregation of size $M$ (finite state) or linear function approximation of dimension $d$ (continuous state).
\end{itemize}
All numbers are averaged over $10$ random seeds, with shaded regions and error bars indicating the standard deviation.

\subsection{Inventory Problem}
We first consider a finite-state inventory MDP. In this task, the state $s$ denotes the net inventory at the beginning of each stage
and takes values in $\gS = \{-b_{\max},\ldots,s_{\max}\}$,
where negative values represent backlogged demand of up to $b_{\max}$ units.
The action $a \in \gA = \{0,\ldots,a_{\max}\}$ denotes the order quantity. The lead time is deterministic and equal to zero, meaning that any placed order is received immediately. Demand $D \sim \mathrm{Poisson}(\lambda)$ is a random variable following a Poisson distribution with an unknown parameter $\lambda$. The stage reward is $r(s,a;D) = p \min\{(s+a)^+, D\} - c a$, and the next state is $s' = \Pi_{[-b_{\max}, s_{max}]}(s + a - D)$. The selling price $p$ and the unit ordering cost $c$ are known constants. The objective is to maximize the total expected profit.

Let $f_{\lambda}(d)$ denote the probability mass function of a Poisson distribution with parameter $\lambda$ evaluated at $d$, with the convention $f_\lambda(d) = 0$ for $d < 0$ and empty sums equal to zero.  
The transition probabilities are defined as follows:
\begin{align*}
    P_{\lambda}(s' | s,a)
    &= 
    \begin{cases}
        \displaystyle \sum_{d=s+a+b_{\max}}^{\infty} f_{\lambda}(d)
        & \text{for } s' = -b_{\max}, \\[6pt]
        f_{\lambda}(s+a-s')
        & \text{for } -b_{\max} < s' < s_{\max}, \\[6pt]
        \displaystyle  \sum_{d=0}^{s+a-s_{\max}} f_{\lambda}(d)
        & \text{for } s'=s_{\max}
    \end{cases}.
\end{align*}

\subsection{Cliff Walking} \label{sec:exp_cw}
  We then consider a stochastic $4{\times}12$ Cliff-Walking grid~\citep{sutton2018reinforcement}
  ($S=48$, $|\mathcal{A}|=4$, $\beta=0.95$), with start in the lower-left and
  goal in the lower-right. Each nonterminal transition has a reward $-1$; entering
  the cliff gives a reward $-50$ and resets the agent to start; reaching the goal
  terminates with a reward $0$. Dynamics depend on a latent weather regime
  $z\in\{0,1,2,3\}$ (\textsc{Clear}, \textsc{Ice-West}, \textsc{Ice-Mid},
  \textsc{Ice-East}), which determines slip probabilities over outcomes
  $\mathcal{O}=\{\mathrm{intended},\mathrm{left},\mathrm{right},\mathrm{back},
  \mathrm{stay}\}$. The transition kernel marginalizes the regime-dependent slip outcome through a deterministic grid update, $P_z(s'| s,a) = \sum_{o\in\mathcal{O}} p_z(o| s,a),\mathbf{1}\{s'=F(s,a,o)\}$, where $p_z(o| s,a)$ is the slip-outcome distribution under regime $z$ and $F(s,a,o)$ is the deterministic grid update (boundary, cliff-reset, and goal rules included).
  
  Under \textsc{Clear}, the short bottom-row path is comparatively safe; under
  \textsc{Ice}, elevated slip at six volatile cells near the cliff makes this
  path risky and often favors upper-row detours. Adjacent ice regimes have
  overlapping weather-feature distributions, so the learner must infer the
  hazardous region from noisy features. Full details are in
  \Cref{app:cliffwalking}.

\subsection{Continuous State MDP} \label{sec:exp_cartpole}
We finally consider the continuous-state CartPole problem~\citep{brockman2016openai}, with state $s = (x, \dot x, \phi, \dot\phi) \in \R^4$, action space $\gA = \{0, 1\}$, reward $r_t = \cos(\phi_t)$, and horizon $T=200$. The transition kernel $P(s' | s, a)$ is given by the standard CartPole rigid-body dynamics integrated by one Euler step of size $\mathrm{d}t = 0.02$. The transition is unknown to the learner; it is observed only through an offline buffer of transitions $\{(s_t, a_t, r_t, s_{t+1})\}_{t=1}^N$ collected by a behavior policy.

The basis $\Phi$ and test functions $\Psi$ are feedforward neural networks trained jointly with the predictor by backpropagating~\Cref{eq:dfl_gradient}. Since the true occupancy \LP is infinite-dimensional, we use a relatively large feature dimension ($d=1024$) for DFL-LP and DFL-QP, and $d \in \{30, 60, 90\}$ for DFL-SA. Because the reward is action-independent, we apply a behavior-cloning reward shaping $r_t + \alpha\,\pi_{\mathrm{bc}}(a_t | s_t)$ with $\alpha=2$ on the LP forward solve while keeping the regret loss on the raw reward; the two-stage baseline shares the same BC-pretrained features and LP-based policy extraction, differing only in that the DFL gradient phase is omitted. Full details, including the Monte Carlo estimators for the LP coefficients and an additional model-based RL baseline, are deferred to~\Cref{app:cartpole}.

\section{Discussion of Experimental Results}
For finite-state tasks, the occupancy \LP has $|\gS||\gA|$
variables and $|\gS|$
flow constraints, giving $36\times31=1116$ variables for Inventory and
$48\times4=192$ for Cliff Walking.
We then compare the methods using training reward, test regret
(lower is better), and total wall-clock training time,
including forward optimization and backpropagation.

\paragraph{Decision quality.}
On all three tasks, the best-performing DFL variants achieve lower mean test regret than the two-stage baseline (\Cref{fig:learning_curve,fig:decision_quality}).
DFL-Sketch improves upon DFL-LP on Inventory and CartPole; on CartPole, it reduces mean regret from approximately $24$ to $9$, consistent with the benefit of smoothing gradients across LP basis changes.
With sufficiently large aggregation sizes, DFL-SA achieves regret comparable to DFL-QP on Inventory and Cliff Walking and lower mean regret on CartPole, while solving substantially smaller optimization problems.
On CartPole, DFL-SA ($d{=}90$) reduces mean regret to approximately $3$, compared with $10$ for DFL-QP and $45$ for two-stage.
The gains over two-stage are modest on Inventory, where the baseline is already competitive.
On Cliff Walking, DFL-SA also achieves lower mean regret under the shifted \textsc{Ice-Mid} test distribution, whereas DFL-LP provides only a small improvement.
Together, these results support sketching and learned aggregation as effective ways to improve the decision quality of LP-based learning across finite and continuous-state MDPs.

\begin{figure}[htbp]
  \includegraphics[width=\linewidth]{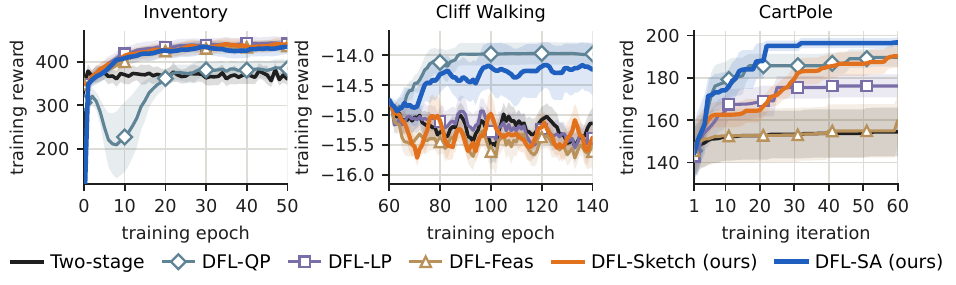}
  \vspace{-22pt}
  \caption{Training reward across epochs on three tasks. DFL-QP and DFL-SA achieve the highest training rewards, while DFL-Sketch and the other LP-based variants are generally competitive with the two-stage baseline.} 
  \label{fig:learning_curve}
\end{figure}

\begin{figure}[htbp]
  \includegraphics[width=\linewidth]{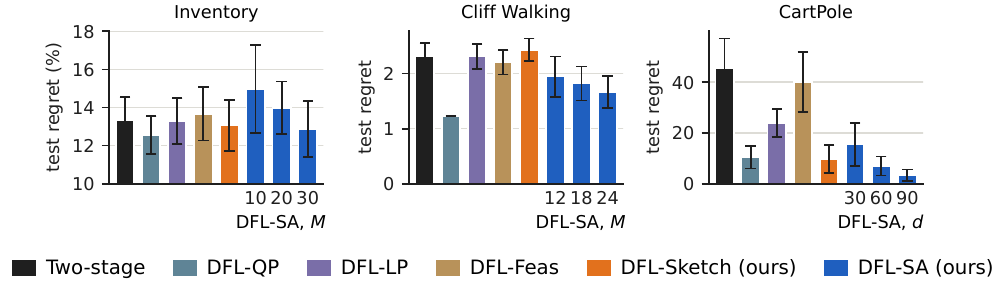}
  \vspace{-22pt}
  \caption{Test regret on three tasks (lower is better). DFL-Sketch improves over DFL-LP on Inventory and CartPole, showing the benefit of smoothing across basis changes. DFL-SA further achieves regret close to DFL-QP at substantially lower cost.}
  \label{fig:decision_quality}
\end{figure}

\begin{figure}[htbp]
  \includegraphics[width=\linewidth]{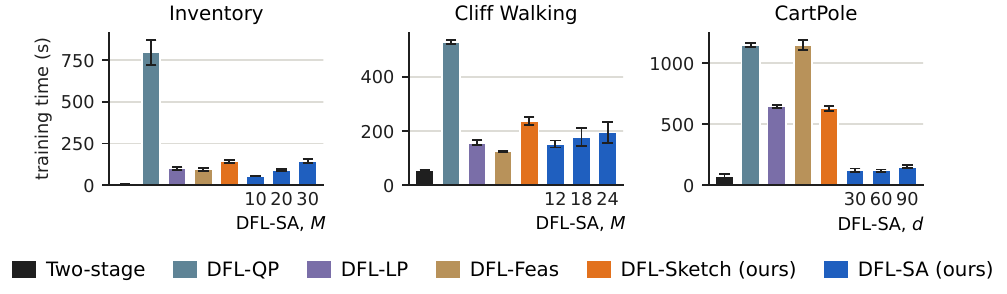}
  \vspace{-22pt}
  \caption{Total training time for different tasks. DFL-QP is the most expensive method because it differentiates through the full KKT system. DFL-Sketch adds overhead from solving multiple sketched \LPs, while DFL-SA reduces this cost through state aggregation.}
  \label{fig:computation_cost}
\end{figure}

\paragraph{Computation cost.}
\Cref{fig:computation_cost} shows the computational benefits of LP-based differentiation and state aggregation.
DFL-LP replaces differentiation through the full KKT system with a smaller basis solve, reducing training time by about $8\times$ relative to DFL-QP on Inventory.
DFL-Sketch combines gradients from multiple sketched LPs to smooth basis changes, with additional computational cost that varies across tasks.
DFL-SA further reduces the size of the basis system through learned aggregation, achieving strong decision quality at substantially lower cost than DFL-QP.
On CartPole, DFL-SA ($d{=}90$) trains about $4\times$ faster than DFL-LP and $7.5\times$ faster than DFL-QP while achieving the lowest mean test regret.
This advantage grows in the larger Inventory experiments: with $|\gS|{=}101$, DFL-SA ($M{=}30$) achieves the lowest mean regret while training approximately $40\times$ faster than DFL-QP.

\paragraph{Effect of the state aggregation size.}
\Cref{fig:decision_quality} shows that increasing the aggregation size reduces mean test regret across the evaluated configurations.
On Inventory, increasing $M$ from $10$ to $30$ lowers regret from $15\%$ to $13\%$, approaching DFL-QP.
Cliff Walking follows a similar trend, with larger aggregation sizes narrowing the gap to DFL-QP.
On CartPole, increasing $d$ from $30$ to $90$ reduces mean regret from $15$ to $3$, illustrating the value of a richer learned representation.
These results suggest that larger aggregate representations retain more information relevant to policy selection, while smaller representations reduce computation.
Aggregation size therefore provides a simple way to balance decision quality and computational cost.
Training time generally increases with aggregation size, although the CartPole measurements are not strictly monotone because wall-clock cost also depends on solver behavior (\Cref{fig:computation_cost}).

\paragraph{Other ablation studies} We further examine scalability to larger Inventory state spaces, sensitivity to the feasibility penalty $\rho$, and the role of the diagonal projection in soft aggregation in \Cref{app:scaling}--\Cref{app:ablation_proj}.
These experiments complement the main results by assessing computational scaling, sensitivity to the surrogate penalty, and the contribution of the MDP-structured aggregation.

%% file: 06_conclusion.tex
\section{Conclusion}
We present a differentiable \LP layer for decision-focused learning in MDPs, with a closed-form gradient requiring only a $|\gS|\times|\gS|$ basis solve.
An augmented Lagrangian surrogate corrects the bias from predicted-dynamics infeasibility, while random row sketching smooths gradients across basis boundaries.
Learned state aggregation reduces the basis size to $M\ll|\gS|$, and function approximation extends the approach to continuous states.
Experiments on Inventory, Cliff Walking, and CartPole show lower regret than two-stage and decision quality comparable to or better than KKT-based DFL at substantially lower computational cost.

%% file: appendix.tex
\section{Appendix}


\subsection{Sample-based estimators for the projected LP}
\label{app:projected_lp}

Recall the projected LP (\ref{eq:projected_lp}) in continuous-state MDPs, with coefficients
\begin{align*}
    A = \Psi^\top (B - \beta P)\Phi \in \R^{m \times d}, \quad
    b = \Psi^\top \gamma \in \R^m, \quad
    (\tilde{r}_\theta)_j = \int_{\gS \times \gA} r_\theta(s, a) \phi_j(s, a) ds da,
\end{align*}
where $(B y)(s) = \int_\gA y(s, a) da$ and $(P y)(s) = \int_{\gS \times \gA} P(s \mid s', a') y(s', a') ds' da'$. We now derive Monte Carlo estimators for each coefficient given offline transitions $\{(s_t, a_t, r_t, s_{t+1})\}_{t=1}^N$ sampled from a behavior policy and $N_0$ initial states $\{s_0^{(k)}\}_{k=1}^{N_0}$ drawn from $\gamma$.

We work with a \emph{behavior-weighted} version of the projected LP. We write the occupancy relative to the data distribution, $y(s,a) = \mu_b(s,a)\,\Phi(s,a)^\top w$, and keep the test functions $\Psi$ unweighted. This avoids dividing by $\mu_b$ in the empirical sums, which is unstable when $\mu_b$ is small; the price is that we solve a behavior-weighted version of~\ref{eq:projected_lp}, not its Lebesgue-weighted counterpart. Under this convention, the estimators for $\Psi^\top P\Phi$, $b$, and $\tilde r_\theta$ below are plain empirical averages. The $B$ block needs extra care, which we explain below.

\paragraph{Estimator for $\Psi^\top P \Phi$.}
Under the behavior-weighted inner product,
\begin{align*}
    (\Psi^\top P \Phi)_{ij}
    &= \int_{\gS \times \gA} \mu_b(s', a') \phi_j(s', a') \left( \int_\gS \psi_i(s) P(s \mid s', a') ds \right) ds' da' \\
    &= \E_{(s', a') \sim \mu_b, s \sim P(\cdot \mid s', a')} \left[ \psi_i(s) \phi_j(s', a') \right].
\end{align*}
Replacing the expectation with its empirical average over the transitions $(s_t, a_t, s_{t+1})$ gives
\begin{align*}
    (\Psi^\top P \Phi)_{ij} \approx \frac{1}{N} \sum_{t=1}^N \psi_i(s_{t+1}) \phi_j(s_t, a_t).
\end{align*}

\paragraph{Estimator for $\Psi^\top B \Phi$.}
Under the same convention,
\begin{align*}
    (\Psi^\top B \Phi)_{ij}
    = \int_\gS \psi_i(s) \sum_{a \in \gA} \mu_b(s,a) \phi_j(s, a) ds
    = \E_{(s,a) \sim \mu_b} \left[ \psi_i(s) \phi_j(s, a) \right]
    \approx \frac{1}{N} \sum_{t=1}^{N} \psi_i(s_t) \phi_j(s_t, a_t),
\end{align*}
which uses only the logged action $a_t$. In our experiments, we instead used the all-action sum
\begin{align*}
    (\Psi^\top B \Phi)_{ij} \approx \frac{1}{N} \sum_{t=1}^{N} \psi_i(s_t) \sum_{a \in \gA} \phi_j(s_t, a),
\end{align*}
which weights the $B$ block by the state marginal $p_b(s)$ instead of $\mu_b(s,a)$. Compared with the consistent estimator, this effectively replaces the discount $\beta$ in each column $(s,a)$ by $\beta\,\pi_b(a \mid s)$. The LP used for CartPole is therefore a heuristic surrogate of the behavior-weighted LP, not an exact estimate of it. The consistent estimator above removes this mismatch; with it, the policy is extracted as $\pi(a \mid s) \propto \pi_b(a \mid s)\,\Phi(s,a)^\top w^*$.

\paragraph{Estimator for $b$.}
Since $b_i = \int_\gS \psi_i(s) \gamma(s) ds = \E_{s_0 \sim \gamma}[\psi_i(s_0)]$, the initial-state samples give
\begin{align*}
    b_i \approx \frac{1}{N_0} \sum_{k=1}^{N_0} \psi_i(s_0^{(k)}).
\end{align*}

\paragraph{Estimator for $\tilde{r}_\theta$.}
Following the same recipe,
\begin{align*}
    (\tilde{r}_\theta)_j
    = \E_{(s, a) \sim \mu_b}\left[ r_\theta(s, a) \phi_j(s, a) \right]
    \approx \frac{1}{N} \sum_{t=1}^N r_\theta(s_t, a_t) \phi_j(s_t, a_t),
\end{align*}
which is again the plain empirical average under the behavior-weighted convention.

The matrix $A$ and vector $b$ depend on the predicted dynamics only through these empirical sums, so backpropagating through~\Cref{eq:dfl_gradient} reduces to backpropagating through the per-sample features $\psi_i, \phi_j$, a standard automatic-differentiation operation when $\Phi, \Psi$ are neural networks.

\subsection{Derive discrete-state soft aggregation from continuous-state function approximation}\label{app:agg_recovery}

\paragraph{Learnable membership construction.}
We parameterize the membership matrix through aggregation parameters $\xi = (E, \Theta_W)$, where $E \in \R^{S \times d_e}$ is a state embedding matrix and $\Theta_W$ maps features to cluster logits. Specifically, we first compute a transition-aware feature for each state-action pair as the expected next-state embedding $\phi(s, a) = \sum_{s'} P(s' \mid s, a) E_{s'} \in \R^{d_e}$, then form a state-level representation by concatenating across all actions: $\bar\phi(s) = [\phi(s,1); \ldots; \phi(s,A)] \in \R^{A d_e}$. This feature encodes how each state interacts with the transition structure under every action. The soft cluster assignment is then obtained by projecting $\bar\phi(s)$ onto $M$ cluster logits via $\Theta_W$ and applying a temperature-scaled softmax:
\begin{align*}
    W_{sm}(\xi) = \frac{\exp\bigl((\Theta_W^\top \bar\phi(s))_m / \tau\bigr)}{\sum_{m'} \exp\bigl((\Theta_W^\top \bar\phi(s))_{m'} / \tau\bigr)}.
\end{align*}
Given the soft membership $W$, we construct a reduced MDP over $M$ abstract states by computing weighted averages of the original MDP quantities. 

We now verify that the basis choices $\phi_{(m, a')}(s, a) = W_{sm} \mathbf{1}[a = a']$ and $\psi_i(s) = W_{si}$ reduce the projected~\ref{eq:projected_lp} to the soft-aggregated LP. The reduction is exact for hard cluster memberships; for soft memberships, the basis-block of the constraint matrix has off-diagonal entries and an exact algebraic match to the standard aggregated flow constraint requires replacing this block by its diagonal projection (equivalently, enforcing the simplex identity $\sum_m W_{sm} = 1$ at the constraint level). We make this step explicit in the derivation below.
\paragraph{Coefficient $b$.} Direct computation gives
\begin{align*}
    b_i = \sum_s \psi_i(s) \gamma_s = \sum_s W_{si} \gamma_s = \hat\gamma_i.
\end{align*}

\paragraph{Coefficient $\tilde{r}_\theta$.} Indexing the basis by $j = (m, a')$,
\begin{align*}
    (\tilde{r}_\theta)_{(m, a')}
    = \sum_{s, a} r_\theta(s, a) W_{sm} \mathbf{1}[a = a']
    = \sum_s W_{sm} r_\theta(s, a')
    = \mu_m \hat r_{m, a'},
\end{align*}
using $\hat r_{m, a'} = \mu_m^{-1} \sum_s W_{sm} r_\theta(s, a')$ from the soft-aggregation construction.

\paragraph{Coefficient $A = \Psi^\top (B - \beta P) \Phi$.} For the $B$ block,
\begin{align*}
    (\Psi^\top B \Phi)_{i, (m, a')}
    = \sum_s W_{si} \sum_a W_{sm} \mathbf{1}[a = a']
    = \sum_s W_{si} W_{sm}.
\end{align*}
For the $P$ block,
\begin{align*}
    (\Psi^\top P \Phi)_{i, (m, a')}
    &= \sum_{s} W_{si} \sum_{s', a} P(s \mid s', a) W_{s'm} \mathbf{1}[a = a'] \\
    &= \sum_{s, s'} W_{si} P(s \mid s', a') W_{s'm}
    = \mu_m \hat P(i \mid m, a'),
\end{align*}
using $\hat P(n' \mid m, a') = \mu_m^{-1} \sum_{s, s'} W_{sm} P(s' \mid s, a') W_{s' n'}$ and renaming the dummy summation indices.

Combining the two blocks,
\begin{align*}
    A_{i, (m, a')}
    = \sum_s W_{si} W_{sm} - \beta \mu_m \hat P(i \mid m, a').
\end{align*}
For hard cluster memberships ($W$ has one-hot rows), $\sum_s W_{si} W_{sm} = \delta_{i, m} \mu_m$, so
\begin{align*}
    A_{i, (m, a')} = \mu_m \big( \delta_{i, m} - \beta \hat P(i \mid m, a') \big) = \mu_m \hat H_{i, (m, a')}.
\end{align*}
For soft memberships, direct projection does \textbf{not preserve} the
standard MDP flow constraints. We therefore replace $W^\top W$
with $\operatorname{diag}(\mu)$ so that the reduced LP describes
a valid aggregated MDP. This replacement is exact for hard
memberships and is an approximation for soft memberships.

After this replacement, the projected LP becomes
\begin{align*}
\max_{w\geq0}\quad
& \sum_{m,a}\mu_m\hat r_{m,a}w_{(m,a)} \\
\text{s.t.}\quad
& \sum_{m,a}\hat H_{i,(m,a)}\mu_m w_{(m,a)}
=\hat\gamma_i
\quad\text{for all }i.
\end{align*}
Assuming $\mu_m>0$, define the aggregated occupancy by
$\hat y(m,a):=\mu_m w_{(m,a)}$.
This change of variables gives the standard aggregated LP
\begin{align*}
\max_{\hat y\geq0}\quad
& \hat r^\top\hat y \\
\text{s.t.}\quad
& \hat H\hat y=\hat\gamma.
\end{align*}
Consequently, lifting an aggregated solution to the original
state space gives
\begin{align*}
y_{\mathrm{disagg}}(s,a)
&=\sum_m W_{sm}w_{(m,a)} \\
&=\sum_m W_{sm}\frac{\hat y(m,a)}{\mu_m}.
\end{align*}

\subsection{Proof for Proposition~\ref{prop_lp}}\label{app:pf_prop_lp}
\begin{proof}
Let $\eta^*$ denote the Lagrangian multiplier for $y \geq 0$. Then, we can write
    \begin{align*}
        L(y^*_\theta) - L(y^*_{\mathrm{true}}) &= -r_{\mathrm{true}}^\top y^*_\theta + (H_{\mathrm{true}} y^*_\theta - \gamma)^\top\nu^* + \rho \|H_{\mathrm{true}} y^*_\theta - \gamma\|^2 + r_{\mathrm{true}}^\top y^*_{\mathrm{true}} \\
        &=(H_{\mathrm{true}}^\top \nu^* - r_{\mathrm{true}})^\top y^*_\theta + \rho \|H_{\mathrm{true}}y^*_\theta - \gamma\|^2\\
        &= (\eta^*)^\top y^*_\theta + \rho \|H_{\mathrm{true}}y^*_\theta - \gamma\|^2.
    \end{align*}
    The second equality is due to strong duality, i.e., $r_{\mathrm{true}}^\top y^*_\mathrm{true} = \gamma^\top \nu^*$. By dual feasibility, we know that $(\eta^*)^\top y^*_\theta \geq 0$, and thus $L(y^*_\theta) - L(y^*_{\mathrm{true}}) \geq 0$. When the equality holds, we have both $(\eta^*)^\top y^*_\theta = 0$ and $\|H_{\mathrm{true}}y^*_\theta - \gamma\|^2 = 0$. With $y^*_\theta \geq 0$ these are the KKT conditions of the true LP, which are sufficient, so $y^*_\theta$ is optimal and $r_{\mathrm{true}}^\top y^*_\theta = r_{\mathrm{true}}^\top y^*_{\mathrm{true}}$. The converse is complementary slackness.
\end{proof}

\section{Details of Experiments}
\subsection{Compute resources, hyperparameters and model architectures}\label{app:hyperparameters}
All experiments in this paper were run on a single shared workstation with two Intel Xeon Gold 6548Y+ processors (128 logical cores total, 2.5 GHz base) and 1.5 TB of system memory.

\Cref{tab:hyperparameters} reports the hyperparameters used in each experiment. All methods on a given task share the same predictor architecture, optimizer, and data; only the gradient (LP, QP, sketch, or none) differs. Numbers are averaged over $10$ random seeds.

\begin{table}[htbp]
\centering
\small
\caption{Key hyperparameters and model architectures for each task.}
\label{tab:hyperparameters}
\begin{tabular}{lccc}
\toprule
& \textbf{Inventory} & \textbf{Cliff Walking} & \textbf{CartPole} \\
\midrule
\multicolumn{4}{l}{\emph{MDP}} \\
State space size $|\gS|$ & $36$ ($s_{\max}=30$, $b_{\max}=5$) & $48$ ($4{\times}12$ grid) & continuous, $\R^4$ \\
Action space size $|\gA|$ & $31$ ($a_{\max}=30$) & $4$ & $2$ \\
Discount $\beta$ & $0.95$ & $0.95$ & $0.99$ \\
Horizon $T$ & N/A & N/A & $200$ \\
\midrule
\multicolumn{4}{l}{\emph{Predictor / features}} \\
Architecture & MLP & MLP & MLP \\
Hidden layers & $(32, 32)$ & $(64,64)$ & $(128, 128)$ \\
Activation & ReLU & ReLU + Tanh & ReLU \\
\midrule
\multicolumn{4}{l}{\emph{Optimization}} \\
Optimizer & Adam & Adam & Adam \\
Learning rate & $10^{-2}$ & $3{\times}10^{-3}$ & $10^{-3}$ \\
Batch size & full batch & $32$ & $4096$ \\
Pretraining epochs & $30$ & $60$ & $70$ \\
Training epochs / iters. & $50$ & $140$ & $60$ \\
Random seeds & $10$ & $10$ & $10$ \\
\midrule
\multicolumn{4}{l}{\emph{DFL surrogate}} \\
Aug.\ Lagrangian $\rho$ (DFL-LP) & $10$ & $0.1$ & $1.0$ \\
Aug.\ Lagrangian $\rho$ (DFL-Sketch) & $0.1$ & $1.0$ & $1.0$ \\
Aug.\ Lagrangian $\rho$ (DFL-QP) & $10^4$ & $0.1$ & $1.0$ \\
QP regularization $\eta$ & $10^{-3}$ & $0.1$ & $10^{-2}$ \\
Sketch ensemble size $K$ & $8$ & $8$ & $8$ \\
Sketch row-keep ratio $\alpha$ & $0.1$ & $0.2$ & $0.2$ \\
LP basis ridge & $10^{-6}$ & none & $10^{-4}$ \\
\midrule
\multicolumn{4}{l}{\emph{State aggregation}} \\
Aggregation sizes $M$ or $d$ & $M \in \{10, 20, 30\}$ & $M \in \{12, 18, 24\}$ & $d \in \{30, 60, 90\}$ \\
Softmax temperature $\tau$ & $0.5$ & $2.0$ & $0.5$ \\
\midrule
\multicolumn{4}{l}{\emph{Data and evaluation}} \\
Train samples & $10$ & $100$ & $500$ trajectories \\
Test samples & $50$ & $100$ & $100$ episodes \\
\bottomrule
\end{tabular}
\end{table}

\subsection{Details for Cliff Walking MDP Task}\label{app:cliffwalking}

  \paragraph{Slip parameterization.}
  The latent weather regime determines a compact slip tensor rather than an
  unstructured transition matrix. The compact tensor has one ambient slot shared
  by all states and six volatile-cell slots that override the ambient slip
  distribution at selected grid cells near the cliff. Each slot contains a
  categorical distribution over the five movement outcomes for each action. This
  keeps the prediction problem structured while still allowing the hazardous
  region of the grid to change with the weather regime.

  \paragraph{Transition construction.}
  For each state-action pair, the compact slip distribution is expanded into the
  full transition kernel by applying the five slip outcomes to the chosen action.
  Each outcome deterministically maps the agent to a candidate next cell after
  enforcing grid boundaries, cliff resets, and goal termination. The probability
  of a next state is the sum of the probabilities of all slip outcomes that land
  in that state. Thus the learned object is not an arbitrary $48\times 4\times 48$
  transition tensor, but a structured slip model that induces the transition
  kernel used by the planner.

  \paragraph{Feature ambiguity and shift.}
  Weather features are generated from regime-specific Gaussian centers. To make
  the shifted test regime difficult to distinguish from the dominant training
  regime, we make the \textsc{Clear} and \textsc{Ice-Mid} centers identical
  (center pair $(0,2)$ with zero gap). Training samples are drawn from the
  skewed prior $15{:}1{:}1{:}1$, whereas test samples are drawn only from
  \textsc{Ice-Mid}. This creates a setting where a predictor can perform well
  on the training distribution while still selecting a poor policy under the
  test dynamics.

  \paragraph{Policy evaluation.}
  All methods plan using the transition kernel induced by their predicted slip
  distribution, but evaluation is always performed in the true MDP for the
  sample. Thus regret measures the downstream cost of selecting a policy under
  misspecified dynamics, not the prediction error itself. The oracle value
  $v^*$ is obtained by solving the occupancy LP with the true transition kernel,
  and the predicted-policy value is computed by evaluating the policy selected
  from $\hat P$ in the true transition kernel.

  \paragraph{State aggregation and sketching.}
  For DFL-SA, soft state aggregation maps the original $48$ states to
  $M \in \{12,18,24\}$ aggregate states before the LP solve, then disaggregates
  the aggregate occupancy measure back to the original state-action space for
  evaluation. For DFL-Sketch, no state aggregation is used. Instead, each
  training step averages decision gradients over $K=8$ randomized LP sketches,
  each retaining an $\alpha=0.2$ fraction of the LP rows.

\Cref{tab:cw_paired} reports an earlier experimental configuration, with a different regret scale and aggregation grid ($M\in\{6,12,24\}$) from~\Cref{fig:decision_quality} ($M\in\{12,18,24\}$). In this earlier configuration, DFL-QP and all listed DFL-SA variants win on all ten paired seeds. These win counts do not establish per-seed wins or statistical significance for the current main-figure configuration.

  \begin{table}[!htbp]
  \centering
  \caption{Paired Cliff Walking test regret over 10 random seeds. Lower regret is
  better. Win counts compare each method against the two-stage baseline on the
  same seed. SE denotes the standard error of the mean.}
  \label{tab:cw_paired}
  \begin{tabular}{lccc}
  \toprule
  Method & Mean regret $\downarrow$ & SE & Wins vs. TS \\
  \midrule
  Two-stage & 14.75 & 0.29 & -- \\
  DFL-LP & 14.14 & 0.18 & 7/10 \\
  DFL-Sketch & 11.34 & 1.44 & 8/10 \\
  DFL-QP & 9.53 & 0.84 & 10/10 \\
  DFL-SA ($M=6$) & 7.21 & 0.15 & 10/10 \\
  DFL-SA ($M=12$) & 7.76 & 0.15 & 10/10 \\
  DFL-SA ($M=24$) & 7.87 & 0.64 & 10/10 \\
  \bottomrule
  \end{tabular}
  \end{table}

\subsection{Details for Continuous MDP Task}\label{app:cartpole}
We expand the experimental setup of~\Cref{sec:exp_cartpole}.

\paragraph{Behavior-cloning reward shaping.} On CartPole, the reward $r_t = \cos(\phi_t)$ is independent of the action, so the LP cannot rank the two actions by reward alone. To break this symmetry, we add a behavior-cloning bonus to the forward solve. A linear head $w_{\mathrm{bc}}$ on top of the state features $\Phi$ is trained to predict the behavior actions by cross-entropy, giving $\pi_{\mathrm{bc}}(a \mid s) = \mathrm{softmax}(w_{\mathrm{bc}}^\top \phi(s))$. The LP forward solve uses the shaped reward $r_t + \alpha\,\pi_{\mathrm{bc}}(a_t \mid s_t)$ with $\alpha = 2$, while the DFL loss~\Cref{eq:dfl_penalty} is still evaluated with the raw reward $r_t$, so the regret being optimized is unchanged.

\paragraph{Two-stage baseline.} On CartPole, the two-stage baseline shares the BC-pretrained features $\Phi$, the shaped reward, and the LP-based policy extraction~\ref{eq:projected_lp} with DFL. The only difference is that the DFL gradient phase is omitted: $\Phi$ is held fixed at its BC-pretrained values throughout. This isolates the contribution of decision-aware feature learning while controlling for the BC pretraining and reward shaping that DFL also receives.

\paragraph{Sample-based LP coefficients.} Given $N$ offline transitions $(s_t, a_t, r_t, s_{t+1})_{t=1}^N$ collected by a behavior policy and $N_0$ initial states $\{s_0^{(k)}\}_{k=1}^{N_0}$ drawn from $\gamma$, the coefficients of~\ref{eq:projected_lp} are computed by the Monte Carlo estimators in~\Cref{app:projected_lp}. The transition kernel $P$ enters only through the empirical next states $s_{t+1}$, so no learned model of $P$ or $r$ is required.

In addition, to address the additional concern of whether DFL beats a more conventional model-based reinforcement learning (MBRL) baseline, and whether such a baseline benefits from receiving the same imitation signal that DFL does, we report two further variants on CartPole:
\begin{itemize}[left=1em, nosep]
  \item \textbf{MBRL}: fit a neural dynamics model $\hat P_\theta$ and reward model $\hat r_\theta$ to the offline trajectories by maximum likelihood, then run DQN~\citep{van2016deep} on the learned model. No BC signal is used anywhere.
  \item \textbf{MBRL + BC shaping}: same pipeline as above, but the DQN reward used in the Bellman target is $r_t + \alpha\, \pi_{\mathrm{bc}}(a_t \mid s_t)$ with $\alpha = 2$, where $\pi_{\mathrm{bc}}$ is trained on the same offline data with the same architecture and pretraining schedule as the BC head used by DFL. The DQN replay is pre-shaped before insertion, so the Q-network learns under the shaped reward, while evaluation on the real environment uses the raw reward.
\end{itemize}

\Cref{tab:mbrl_baseline} reports best-so-far CartPole rewards. The additional MBRL runs use the biased behavior policy ($\theta$-weight $0.013$), $100$ offline trajectories, $60$ outer iterations, a $(32,32)$ Q-network, $400$ world-model epochs, and $5$ seeds. The main LP experiments use the configuration in~\Cref{tab:hyperparameters}. Thus this table is a descriptive comparison across the reported settings, not a matched-data ablation establishing a causal effect of BC shaping. The MBRL rows have lower observed means and larger across-seed standard deviations.

\begin{table}[!htbp]
\centering
\small
\caption{CartPole best-so-far reward (mean $\pm$ standard deviation). The MBRL rows use $5$ seeds and the additional configuration described above; the table does not establish a matched-data comparison.}
\begin{tabular}{lcc}
\toprule
Method & Best-so-far reward \\
\midrule
DFL-LP                                  & $170.5 \pm 32.7$\\
DFL-QP                                  & $192.4 \pm 12.6$\\
Two-stage (BC + LP, no DFL gradient)    & $168.1 \pm 24.8$\\
\midrule
MBRL (no BC anywhere)                   & $136.8 \pm 66.6$ \\
MBRL + BC reward shaping ($\alpha = 2$) & $127.3 \pm 61.3$\\
\bottomrule
\end{tabular}

\label{tab:mbrl_baseline}
\end{table}

\subsection{Scaling to larger inventory MDPs}\label{app:scaling}
To probe scalability beyond the $|\gS| = 36$ inventory MDP of~\Cref{sec:exp} ($|\gS||\gA| = 36 \times 31 = 1116$ occupancy variables), we increase the inventory state size to $|\gS| = 51$ and $|\gS| = 101$ while keeping $|\gA| = 31$, giving $|\gS||\gA| = 1581$ and $3131$ occupancy variables, respectively. \Cref{tab:scaling_51,tab:scaling_101} report test regret and total training time over $10$ seeds.

\begin{table}[!htbp]
\centering
\small
\caption{Larger inventory MDP, $|\gS| \times |\gA| = 51 \times 31 = 1581$ (mean $\pm$ standard error over $10$ seeds).}
\label{tab:scaling_51}
\begin{tabular}{lcc}
\toprule
Method & Test regret $\downarrow$ & Total (s) \\
\midrule
Two-stage & $24.537 \pm 4.573$ & $68.2 \pm 11.2$ \\
DFL-QP & $19.121 \pm 2.156$ & $3649.4 \pm 165.5$ \\
DFL-LP & $23.056 \pm 3.294$ & $972.0 \pm 34.0$ \\
DFL-Feas & $22.834 \pm 3.414$ & $871.0 \pm 15.7$ \\
DFL-Sketch (Ours) & $19.434 \pm 2.363$ & $897.7 \pm 33.2$ \\
DFL-SA ($M = 10$) (Ours) & $22.659 \pm 2.589$ & $226.5 \pm 8.0$ \\
DFL-SA ($M = 20$) (Ours) & $21.159 \pm 2.883$ & $348.5 \pm 11.4$ \\
DFL-SA ($M = 30$) (Ours) & $19.249 \pm 2.762$ & $847.8 \pm 15.1$ \\
\bottomrule
\end{tabular}
\end{table}

\begin{table}[!htbp]
\centering
\small
\caption{Larger inventory MDP, $|\gS| \times |\gA| = 101 \times 31 = 3131$ (mean $\pm$ standard error over $10$ seeds).}
\label{tab:scaling_101}
\begin{tabular}{lcc}
\toprule
Method & Test regret $\downarrow$ & Total (s) \\
\midrule
Two-stage & $53.788 \pm 7.710$ & $227.3 \pm 6.4$ \\
DFL-QP & $44.545 \pm 5.823$ & $34791.0 \pm 756.4$ \\
DFL-LP & $40.690 \pm 4.946$ & $19915.3 \pm 527.1$ \\
DFL-Feas & $62.904 \pm 8.230$ & $19832.6 \pm 618.7$ \\
DFL-Sketch (Ours) & $48.591 \pm 5.718$ & $13790.6 \pm 513.4$ \\
DFL-SA ($M = 10$) (Ours) & $42.568 \pm 5.814$ & $540.6 \pm 3.8$ \\
DFL-SA ($M = 20$) (Ours) & $43.459 \pm 5.432$ & $605.2 \pm 6.0$ \\
DFL-SA ($M = 30$) (Ours) & $39.187 \pm 5.198$ & $862.8 \pm 12.2$ \\
\bottomrule
\end{tabular}
\end{table}

The main takeaways are: (i) at $|\gS| = 51$, DFL-SA ($M{=}30$) closely matches the mean regret of DFL-QP ($19.249$ vs.\ $19.121$) while training $4.3\times$ faster; (ii) at $|\gS| = 101$, DFL-SA ($M{=}30$) achieves the lowest mean regret of all methods and trains $40.3\times$ faster than DFL-QP and $23.1\times$ faster than DFL-LP; and (iii) across every inventory size, the best DFL-SA variant achieves lower mean regret than the two-stage baseline. The relative speedup of DFL-SA thus grows with the original state-space size, consistent with the smaller basis-solve dimension, although wall-clock time also includes matrix construction and forward optimization.

\subsection{Sensitivity to the feasibility penalty $\rho$}\label{app:rho_sensitivity}
We select $\rho$ separately for each method by validation performance. DFL-QP requires a larger $\rho$ on the inventory MDP because its quadratic regularization changes the scale and sensitivity of the resulting occupancy solution, so a stronger feasibility penalty is needed to balance the feasibility and decision-quality gradients. \Cref{tab:rho_sensitivity} reports a sensitivity analysis for DFL-QP on the inventory MDP, showing both test regret and the flow-constraint violation MSE.

\begin{table}[!htbp]
\centering
\small
\caption{Sensitivity of DFL-QP on the inventory MDP to the penalty weight $\rho$ (mean $\pm$ standard error over $10$ seeds).}
\label{tab:rho_sensitivity}
\begin{tabular}{lcc}
\toprule
$\rho$ & Test regret $\downarrow$ & Flow-constraint MSE \\
\midrule
$1$ & $15.629 \pm 2.296$ & $23.98$ \\
$10$ & $14.845 \pm 1.877$ & $23.07$ \\
$10^2$ & $14.124 \pm 1.808$ & $21.26$ \\
$10^3$ & $12.819 \pm 1.185$ & $18.00$ \\
$10^4$ (used in the paper) & $12.552 \pm 0.991$ & $17.82$ \\
$10^5$ & $13.305 \pm 1.458$ & $17.26$ \\
\bottomrule
\end{tabular}
\end{table}

As $\rho$ increases, the flow-constraint MSE decreases monotonically, while the regret improves up to $\rho = 10^4$ and slightly worsens at $\rho = 10^5$. The performance is relatively stable for $\rho \in [10^3, 10^5]$, and we select $\rho = 10^4$ by validation regret.

\subsection{Ablation of the diagonal projection in soft aggregation}\label{app:ablation_proj}
\Cref{sec:state_agg} replaces the $W^\top W$ block of the exactly projected constraint matrix by its diagonal mass approximation $\mathrm{diag}(\mu)$ to obtain an MDP-structured aggregated LP. \Cref{tab:projection_ablation} compares, on the inventory MDP with $M = 10$: (i) \emph{Hard}, non-differentiable hard clustering (one-hot memberships, for which the reduction is exact but no gradient flows through the membership matrix); (ii) \emph{Proj-LP}, the exactly projected LP without the diagonal step; and (iii) \emph{Proj-MDP}, the MDP-structured construction used by DFL-SA. ``Infeasible train solves'' is the fraction of training LP solves in which the reduced solution corresponds to an occupancy that the induced policy cannot realize.

\begin{table}[!htbp]
\centering
\small
\caption{Ablation of the diagonal projection on the inventory MDP ($M = 10$; mean $\pm$ standard error over $10$ seeds).}
\label{tab:projection_ablation}
\begin{tabular}{lcc}
\toprule
LP layer & Test regret $\downarrow$ & Infeasible train solves \\
\midrule
Hard & $63.024 \pm 4.911$ & $0.0\,\%$ \\
Proj-LP (no diagonal step) & $24.147 \pm 5.758$ & $61.8\,\%$ \\
Proj-MDP (DFL-SA) & $14.966 \pm 2.319$ & $0.0\,\%$ \\
\bottomrule
\end{tabular}
\end{table}
The MDP-structured construction has a lower reported diagnostic rate and lower mean regret in this experiment. Feasibility in the reduced MDP does not by itself establish feasibility of a lifted occupancy in the original MDP..

\section{Additional Related Works}\label{app:related_works}
\paragraph{State Aggregation in Markov Decision Processes.}
In MDP, the state space grows exponentially with the number of state variables, a phenomenon commonly referred to as the \emph{curse of dimensionality}. A vast literature has addressed the challenge of exploding state spaces in large-scale MDPs through dimensionality reduction \cite{Bertsekas_IEEE_2019}, approximate dynamic programming \cite{Powell_ADP_2007}, simulation-based algorithms \cite{Chang_SAMDP_2013}, and deep reinforcement learning \cite{Sutton_RL_2018}, along with many recent advances \cite{Bertsekas_IEEE_2019,Bertsekas_RLOC_2020,Dewanto_RL_2020,Abel_2019,NEURIPS2021_454cecc4,xue2023graph}. However, most of these approaches do not integrate naturally within the DFL framework, as the non-smooth operators underlying traditional solvers prevent end-to-end gradient-based training. One promising direction compatible with DFL is differentiable \emph{state aggregation}, in which similar states are mapped to a shared representation to construct a smaller, more tractable MDP while preserving gradient flow.

The foundational concept of state aggregation has long been a classical tool in approximate dynamic programming and reinforcement learning \cite{bertsekas1996neuro}. Broadly speaking, it reduces the dimensionality of large MDPs by collapsing groups of states into abstract representations. The central challenge is designing aggregations that preserve sufficient structural information to recover near-optimal policies. This has motivated a variety of approaches, including homomorphism-based aggregation that preserves transition and reward equivalence \citep{ravindran2003smdp}, projection-based methods that approximate the value or occupancy measure linear program on a reduced basis \citep{de2003linear}, and deep representation-learning methods for latent-state construction \citep{gelada2019deepmdp}.

Most closely related to our work is the literature on learnable soft aggregation. \citet{singh1994reinforcement} introduced soft state aggregation for reinforcement learning using a probabilistic membership matrix over states. More recently, a variety of representation-learning approaches have been proposed to learn latent state abstractions \citep{zhang2021learning}. However, prior work primarily focuses on value-function approximation within a fixed MDP. In contrast, our approach incorporates a learnable, differentiable aggregation layer directly within a predicted MDP and optimizes it end-to-end, thereby connecting representation learning with decision-focused learning.

%% file: checklist.tex
\section*{NeurIPS Paper Checklist}

\begin{enumerate}

\item {\bf Claims}
    \item[] Question: Do the main claims made in the abstract and introduction accurately reflect the paper's contributions and scope?
    \item[] Answer: \answerYes{} 
    \item[] Justification: Please see~\Cref{sec:introduction}.

\item {\bf Limitations}
    \item[] Question: Does the paper discuss the limitations of the work performed by the authors?
    \item[] Answer: \answerYes{} 
    \item[] Justification: The \LP gradient is inherently discontinuous at basis boundaries, which is a hard problem to solve in practice.

\item {\bf Theory assumptions and proofs}
    \item[] Question: For each theoretical result, does the paper provide the full set of assumptions and a complete (and correct) proof?
    \item[] Answer: \answerYes{} 
    \item[] Justification: The closed-form \LP-layer Jacobian and gradient (\Cref{eq:jaco,eq:dfl_gradient} in~\Cref{sec:lp_grad}) are derived under explicitly stated non-degeneracy and fixed-basis assumption.

    \item {\bf Experimental result reproducibility}
    \item[] Question: Does the paper fully disclose all the information needed to reproduce the main experimental results of the paper to the extent that it affects the main claims and/or conclusions of the paper (regardless of whether the code and data are provided or not)?
    \item[] Answer: \answerYes{} 
    \item[] Justification: Please see~\Cref{tab:hyperparameters}.

\item {\bf Open access to data and code}
    \item[] Question: Does the paper provide open access to the data and code, with sufficient instructions to faithfully reproduce the main experimental results, as described in supplemental material?
    \item[] Answer: \answerYes{} 
    \item[] Justification: An anonymous code repository containing the implementations of all methods (DFL-LP, DFL-Sketch, DFL-SA, DFL-QP, two-stage), the three benchmark tasks (Inventory, Cliff Walking, CartPole), and scripts to reproduce all reported figures is provided via the anonymous URL given in the last sentence of the Abstract.

\item {\bf Experimental setting/details}
    \item[] Question: Does the paper specify all the training and test details (e.g., data splits, hyperparameters, how they were chosen, type of optimizer) necessary to understand the results?
    \item[] Answer: \answerYes{} 
    \item[] Justification: Task setup, data generation, train/test splits, and evaluation protocols are described in~\Cref{sec:exp}; the full list of predictor architectures, optimizers, learning rates, augmented-Lagrangian and sketch hyperparameters, and aggregation sizes is reported in~\Cref{tab:hyperparameters} of~\Cref{app:hyperparameters}.

\item {\bf Experiment statistical significance}
    \item[] Question: Does the paper report error bars suitably and correctly defined or other appropriate information about the statistical significance of the experiments?
    \item[] Answer: \answerYes{} 
    \item[] Justification: All reported numbers are averaged over $10$ paired random seeds, with shaded regions and error bars in~\Cref{fig:learning_curve,fig:decision_quality,fig:computation_cost} indicating one standard deviation across seeds, as stated at the start of~\Cref{sec:exp}.

\item {\bf Experiments compute resources}
    \item[] Question: For each experiment, does the paper provide sufficient information on the computer resources (type of compute workers, memory, time of execution) needed to reproduce the experiments?
    \item[] Answer: \answerYes{} 
    \item[] Justification: Per-method total wall-clock training times for each of the three tasks are reported in~\Cref{fig:computation_cost} and discussed in~\Cref{sec:exp}.
    
\item {\bf Code of ethics}
    \item[] Question: Does the research conducted in the paper conform, in every respect, with the NeurIPS Code of Ethics \url{https://neurips.cc/public/EthicsGuidelines}?
    \item[] Answer: \answerYes{} 
    \item[] Justification: The research conforms with the NeurIPS Code of Ethics. It does not involve human-subject experiments, personal or sensitive data, deception, harmful deployment, or unsafe data collection. All datasets and methods used are appropriately documented, and the work is intended for legitimate scientific purposes.

\item {\bf Broader impacts}
    \item[] Question: Does the paper discuss both potential positive societal impacts and negative societal impacts of the work performed?
    \item[] Answer: \answerYes{} 
    \item[] Justification: The proposed methods can improve decision quality in inventory management and any sequential decision-making problem modeled as an MDP with predicted parameters.
    
\item {\bf Safeguards}
    \item[] Question: Does the paper describe safeguards that have been put in place for responsible release of data or models that have a high risk for misuse (e.g., pre-trained language models, image generators, or scraped datasets)?
    \item[] Answer: \answerNA{} 
    \item[] Justification: The paper poses no such risks.

\item {\bf Licenses for existing assets}
    \item[] Question: Are the creators or original owners of assets (e.g., code, data, models), used in the paper, properly credited and are the license and terms of use explicitly mentioned and properly respected?
    \item[] Answer: \answerYes{} 
    \item[] Justification: All third-party assets are cited at point of use: the Cliff Walking environment~\citep{sutton2018reinforcement}, the CartPole environment from OpenAI Gym~\citep{brockman2016openai}, the differentiable QP solver \texttt{qpth}~\citep{amos2017optnet}, and prior KKT-based DFL for MDPs~\citep{wang2021learning}. 

\item {\bf New assets}
    \item[] Question: Are new assets introduced in the paper well documented and is the documentation provided alongside the assets?
    \item[] Answer: \answerYes{} 
    \item[] Justification: Please see the anonymous URL in the last sentence of Abstract.

\item {\bf Crowdsourcing and research with human subjects}
    \item[] Question: For crowdsourcing experiments and research with human subjects, does the paper include the full text of instructions given to participants and screenshots, if applicable, as well as details about compensation (if any)? 
    \item[] Answer: \answerNA{} 
    \item[] Justification: The paper does not involve crowdsourcing nor research with human subjects.

\item {\bf Institutional review board (IRB) approvals or equivalent for research with human subjects}
    \item[] Question: Does the paper describe potential risks incurred by study participants, whether such risks were disclosed to the subjects, and whether Institutional Review Board (IRB) approvals (or an equivalent approval/review based on the requirements of your country or institution) were obtained?
    \item[] Answer: \answerNA{} 
    \item[] Justification: The paper does not involve crowdsourcing nor research with human subjects.

\item {\bf Declaration of LLM usage}
    \item[] Question: Does the paper describe the usage of LLMs if it is an important, original, or non-standard component of the core methods in this research? Note that if the LLM is used only for writing, editing, or formatting purposes and does \emph{not} impact the core methodology, scientific rigor, or originality of the research, declaration is not required.
    \item[] Answer: \answerNA{}{} 
    \item[] Justification: The core method development in this research does not involve LLMs as any important, original, or non-standard components.

\end{enumerate}